\documentclass[runningheads]{llncs}
\usepackage{graphicx}
\usepackage{tikz}
\usepackage{comment}
\usepackage{amsmath,amssymb} % define this before the line numbering.
\usepackage{color}

\usepackage[accsupp]{axessibility}  % Improves PDF readability for those with disabilities.

\usepackage{mathtools} % for dot product \inp
\DeclarePairedDelimiterX{\inp}[2]{\langle}{\rangle}{#1, #2} % for inner product
\newtheorem{prop}{Proposition}[section]
\begin{document}
% \linenumbers
\pagestyle{headings}
\mainmatter
\def\ECCVSubNumber{11}  %●Insert your submission number here

\title{Flood Inflow Forecast Using $\ell_2$-norm Ensemble Weighting Sea Surface Feature
} % Replace with your title

\begin{comment}
\titlerunning{ECCV-22 submission ID \ECCVSubNumber} 
\authorrunning{ECCV-22 submission ID \ECCVSubNumber} 
\author{Anonymous ECCV submission}
\institute{Paper ID \ECCVSubNumber}
\end{comment}
%******************

%\begin{comment}
\titlerunning{Flood Inflow Forecast $\ell_2$-norm Ensemble Weighting Sea Surface Feature}
\authorrunning{T. Yasuno, M. Amakata, J. Fujii, M. Okano, R. Ogata} 
\author{Takato Yasuno\inst{1}\orcidID{0000-0002-4796-518X} \and
Masazumi Amakata\inst{1} \and
Junichiro Fujii\inst{1} \and
Masahiro Okano\inst{1} \and
Riku Ogata\inst{1}}
\institute{Yachiyo Engineering, Co.,Ltd. 5-20-8 Asakusabashi, Koto-ku, Tokyo, Japan 
\email{tk-yasuno@yachiyo-eng.co.jp}\\
}
%\end{comment}
%******************
\maketitle

\begin{abstract}
  It is important to forecast dam inflow for flood damage mitigation. The hydrograph provides critical information such as the start time, peak level, and volume. Particularly, dam management requires a 6-h lead time of the dam inflow forecast based on a future hydrograph. The authors propose novel target inflow weights to create an ocean feature vector extracted from the analysed images of the sea surface. We extracted 4,096 elements of the dimension vector in the fc6 layer of the pre-trained VGG16 network. Subsequently, we reduced it to three dimensions of t-SNE. Furthermore, we created the principal component of the “sea temperature weights” using PCA. We found that these weights contribute to the stability of predictor importance by numerical experiments. As base regression models, we calibrate the least squares with kernel expansion, the quantile random forest minimized out-of-bag error, and the support vector regression with a polynomial kernel. When we compute the predictor importance, we visualize the stability of each variable importance introduced by our proposed weights, compared with other results without weights. We apply our method to a dam at Kanto region in Japan and focus on the trained term from 2007 to 2018, with a limited flood term from June to October. We test the accuracy over the 2019 flood term. Finally, we present the applied results and further statistical learning for unknown flood forecast.
\end{abstract}

\section{Introduction}
\subsection{Flood Dam Inflow Forecast for Heavy and Anomalous Rainfall}
For the past decade, flood damage has become a social problem owing to unexperienced weather conditions arising from climate change. An immediate response to extreme rain and flood situations is important for the mitigation of casualties and economic losses and for faster recovery. 
% 2020.7 extreme rainfall and flood damage
Recently, a record of heavy rain was from 1 h to 72 h during the period of "heavy rain in July, 2nd year of Reiwa" (2020, July 3 to 31). In addition to Kyushu, Gifu, and Nagano prefectures have set new records for unexperienced highest rainfall.
Regarding large-scale atmospheric flow and marine features, the westerlies (subtropical jet stream) near Japan have continued to meander since 2020 early July.
Owing to the increased westerly meandering over the western Eurasian continent (Silk Road teleconnection), the inflow of water vapor to the west coast of Kyushu was extremely large, and "linear precipitation zones" were concentrated in Kyushu, making a significant contribution to the total precipitation.
Early morning on July 4th, heavy rains occurred mainly in the Amakusa region and Kuma region of the Kumamoto prefecture. There was remarkably heavy rainfall of over 200 mm in the area with rainfall for 3 h.
A linear precipitation zone occurred from August 12 to 14 in the northern part of Kyushu and the Chugoku region, and "information on remarkably heavy rain" was announced nine times. Particularly, from the early morning to the dawn of August 14, extremely heavy rain continued owing to the linear precipitation zone in the northern part of Kyushu.

The use of precipitation forecast outputs may influence the accuracy of dam inflow prediction for flood damage mitigation. Particularly, forecasting 6 h ahead is critical and legally requested for the basic policy strengthening flood control function of an existing dam (Ministry of Land Infrastructure, 2019). To prevent social losses due to incoming heavy rain and typhoons, we attempt to predict flood situations at least 6 h beforehand. The dam manager can inform downstream residents of the hazardous flood scenario through an announcement. The elderly and children can then escape to safety facilities. Therefore, forecasting precipitation 6 h ahead of time is crucial for mitigating flood damage and ensuring the safety of people downstream. 

\subsection{Related works}
\subsubsection{Statistical Postprocessing Ensemble Forecast}
For the past two decades, there have been numerous challenges in the formulation of a methodology for the statistical postprocessing of weather forecasts using available data sources (Wilks and Hemri,~2018). Vannitsem et al.(2020) overview the probabilistic postprocessing methods to classify parametric and non-parametric approaches, where it is easy to understand each position on the two axes of implementability and flexibility. The first axis indicates whether it is easy or difficult to implement the tuning parameter choice, available software, and model size. The second axis characterizes the model adaptability to different outputs, and the complexity of representable relationships between the input predictors and target output (Vannitsem et al.,~2020). 
The origin of the parametric method is the Ensemble Model Output Statistics (EMOS) (Gneiting et al.,~2005) using the Gaussian regression model. For higher applicability, there are many state-of-the-art methods such as the Bayesian Model Averaging (BMA) (Raftery et al.,~2005), Nonhomogenous boosting predictor selected EMOS (Messner et al.,~2017), D-vin Copula incorporating dependence structure (Moller et al.,~2018), Bivariate Gaussian models in a distributional regression (Lang et al.,~2019), and distributional regression forest (Schlosser et al.,~2019).

However, non-parametric methods have a widespread interest owing to their flexibility and high-performance. For example, the Quantile regression (Bremness et al.~2004), Analog ensemble (Delle Monache et al.,~2013), Neural network-based quantile regression (Bremness et al.,~2020), Member-by-member (Van Schaeybroeck et al.,~2015), and Quantile regression forest (QRF) (Taillardat et al.,~2016). 
The authors build a practical algorithmic model inspired by non-parametric approaches. We then propose a powerful ensemble forecast method with both higher flexibility and easier implementability for 6-h lead time dam inflow guidance.   

\subsubsection{Sea Surface Temperature Feature}
Sea surface temperature (SST) is a fundamental physical variable for understanding,
Quantifying, and predicting complex interactions between the ocean and atmosphere (O'Carroll et al., 2019). 
Numerical weather prediction (NWP) uses current conditions as input into mathematical models of the atmosphere to predict the weather. The SST affects the behaviour of the overlying atmosphere and vice versa; thus, daily analyses of the SST are required by operational NWP systems (Beggs, 2010). Two-way air–sea coupled weather prediction models have been developed, such as the ECMWF Integrated Forecasting System (IFS)10 (Williams et al., 2015). Operational centres also issue seasonal forecasts from several weeks to months (Balmaseda et al., 2009). Most seasonal forecasting systems are based on coupled ocean-atmosphere circulation models that predict the SST.

The Japan Meteorological Agency (JMA) has been operating a global ocean data assimilation system since 1995 to monitor El Niño and Southern Oscillation (ENSO) conditions. The latest system (Tsujino et al., 2010) is the MOVE (Multivariate Ocean Variational Estimation) / Meteorological Research Institute Community Ocean Model - Global version 2. The output, along with the atmospheric analysis, is also used as an initial condition of a coupled ocean-atmosphere model for the ENSO prediction and seasonal forecast. The sea surface forcing is based on 6-hourly data from the Japanese 55-year Re-Analysis (Kobayashi et al., 2015).
There are many historical observed SST data and global scale 20 km resolution of statistical modelling ocean-atmosphere for numerical forecasts. However, meso scale 5 km resolution of dam inflow forecast incorporating the SST feature is unavailable to the best of our knowledge.      

\subsubsection{Spatially Local Weighted Regression}
Starting the locally weighted regression (Cleveland, et al., 1988), we can implement the generalized least squares method to set a diagonal matrix gathering non-negative weights.
Here, the weights are assigned by a kernel function that contains a local bandwidth parameter controlling the size of the neighbourhood. 
Modelling the conversion from wind to power is a nonlinear regression problem, possibly with time-varying parameters (Pinson et al., 2018). An example has been illustrated as the outcome of fitting local polynomial regression models for the conversion of wind speed to power generation. Unlike the data modelling approach, the weighted target variable has been widely used on the algorithmic modelling approach. For the past two decades, many studies have focused on the relationship between the vegetation feature and atmospheric precipitation. To model the spatial climate-vegetation relationship (Brunsdon, et al, 1998), statisticians should consider the phenomenon of non-stationarity across land cover or vegetation type (Foody, 2003; Propastin et al., 2008). 

These studies are based on the weighted least squares regression method that is restricted data modelling requiring some assumptions. The parametric assumption makes their implementation easier than in non-parametric models. However, the state-of-the-art weighted least squares method could not represent any hydrograph for future dam inflow. This is because it is not a trend curve but a complex triangle shaped with skewness and kurtosis.   
Furthermore, there is no practical method incorporating ocean feature weights of critical boundary condition from sea surface temperature analysed images for dam inflow forecasts.   

\section{Method}
We propose a novel weights of target inflow that can create the sea temperature feature vector extracted from SST analysed images.
Furthermore, we calibrate three base regression models for $\ell_2$-normalized ensemble forecasting dam inflow.

\subsection{Ensemble Regression with Sea Surface Feature}
We calibrate typical regression models based on surrogated trees and the polynomial kernel function. We focused on the Random Forest (RF) regression minimized out-of-bag (OOB) error, and the Support Vector Regression with the polynomial kernel.
We apply these regression models to two forecast problem with a 6-h lead time of the basin rainfall and dam inflow.  
 When we compute the OOB predictor importance, we can visualize the stability of each variable importance introduced with our proposed weights, compared with other results without the weights.
\subsubsection{Base Regressions with SST-Weights}
Statistical modelling is subdivided into two cultures by Breiman (2001). The data modelling culture begins by assuming a stochastic data model that contains the probability distribution with parameters and random noise. The values of the parameters are estimated from the data. The former model validation is yes-no using goodness-of-fit tests and residual examinations. For example, linear and logistic regression. 
In contrast, the algorithmic modelling culture considers the inside of the box complex and the unknown without any stochastic assumption. This approach finds a function : an algorithm that operates on the input to predict the target variable. The later mode validation is measured by predictive accuracy. For example, decision trees, neural nets, random forest, and support vector machine. 

This study applies our proposed ocean weights to three statistical regression models including the least squares model with kernel expansion (Kernel Reg), the quantile regression forest minimizing out-of-bag error (RFoob), and the support vector regression with linear/polynomial kernels (SVMpoly).
We calibrate three weighted base regressions using Bayes optimization with respect to their hyperparameters. Kernel Regression (Rahimi, A., et al., 2008; Le, Q., et al., 2013; Huang, P. S., et al., 2014) has hyperparameters that consist of the learner(least squares, svm), kernel scale, lambda, number of expansion dimensions, and epsilon.
RFoob (L.~Breiman, 2001;~N.~Meinshausen, 2006;~T.~Hastie et al., 2009) has hyperparameters that consist of the minimum leaf size and number of tree size. The number of grown trees is possible to set to 100 for the accuracy and running speed. 
SVMpoly(B.~Scholkopf, et al., 2000;~Fan, R.-E.,et al.,~2005;~Kecman V., T.,et al.,~2005) has hyperparameters that include the box constraint, epsilon, kernel function(linear, polynomial), and polynomial order (useful range from 2.00 to 2.97). 

\subsubsection{$\ell_2$-normalized Ensemble Forecast}
Ensemble learning can be broken down into two tasks: developing a population of base learners from the training data, and then combining them to the composite predictor (T.~Hastie,et al.,~2009). Previously, we presented the statistical modelling of a population of base learners such as Kernel Reg, RFoob, and SVMpoly. Additionally, we formulated their combination to achieve a higher performance predictor, unlike averaging. In the case of the dam inflow and basin rainfall forecast, the equalized weights, one third averaging as a stacking ensemble was not effective as far as we implemented numerical experiments. Hydrological observed data have a characteristic where the boundary condition suddenly changes to the non-stationary phase and observed meteorological data as an initial condition contains uncertain signals and noise accumulated to chaotic phenomena. The variance reduction of base model outputs is necessary to build a powerful ensemble model beyond simple averaging. Based on the work by Hwang et al.(2019), the temperature and rainfall forecast problem was successfully improved using $\ell_2$-normalized ensemble modelling.
This study proposes an $\ell_2$-normalized ensemble method where three base regressors outputs are combined with each output's $\ell_2$-norm denominator, respectively.  
The three base regression prediction outputs are represented as $\hat{y}^{Kernel}$,~$\hat{y}^{RFoob}$,~$\hat{y}^{SVM}$. $N$ is the number of test input data at the flood term. The $\ell_2$-normalized target output is represented as follows.
\begin{equation}
||\hat{y}||_2 = \sqrt{{\hat{y}^T} {\hat{y}} } = \sqrt{\sum_{t=1}^N (\hat{y}_t)^2}
\end{equation}
The $\ell_2$-normalized ensemble forecast is formulated in the next equation. Here, the three term coefficients should satisfy $\sum_{m=1}^3 a_m =1$,~$a_m \geq 0$.
\begin{equation}
\hat{y}^{\ell_2{-normalized Forecast}} = 
a_1 \frac{\hat{y}^{Kernel}}{||\hat{y}^{Kernel}||_2/N} \nonumber 
\end{equation}
\begin{equation}
 + a_2 \frac{\hat{y}^{RFoob}}{||\hat{y}^{RFoob}||_2/N}
 + a_3 \frac{\hat{y}^{SVM}}{||\hat{y}^{SVM}||_2/N}
\end{equation}
For example, when we computed the 6-h lead time dam inflow forecast ocean weighted base regressions, we obtained each value of $\ell_2$-norm$/N$ as 0.3120, 0.7522, 0.7900. It is possible to use three coefficients $a_1=a_2=a_3=\frac{1}{3}$.   

When the previous $\ell_2$-normalized ensemble method is applied to the entire flood term of the test data, there may be a bias gap between the ensemble forecasts and the test actual value. This is because some flood terms occur at an extreme peak level, whereas others occur at a low peak level. In summary, the variance of each flood term is not always steady but a complex shape that features different peak levels, high and low multiple peaks, skewness and kurtosis, and inflow volume represents the area of the hydrograph. Therefore, the dam inflow ensemble forecast must independently normalize each flood term, respectively. The authors propose a novel batch term of the $\ell_2$-normalized ensemble as the next algorithm. 
\begin{table}[h]
\begin{center}
\begin{tabular}{l}
\hline 
\textbf{Algorithm1}: Batch-term $\ell_2$-normalized Ensemble \\
\hline 
\textbf{input} $B$~flood terms~$\{ y_{b,t} \}_{t=1}^{N_{b}} (b=1,...,B)$,\\
~$\sum_{b=1}^{B} N_{b}=N$, \\
~$\sum_{mdl=1}^{3} \alpha_{mdl}^{b}=1,$~$\alpha_{mdl}^{b}\geq 0$, \\
~$||\hat{y}_b^{Kernel}||_2 \neq ||\hat{y}_b^{RFoob}||_2 \neq ||\hat{y}_b^{SVM}||_2 $, \\
$\hat{y}_b^{B \ell_2 {-} Ensemble} \leftarrow \emptyset_{N_b}$, \\
\textbf{for} $mdl \in \{1,2,3 \} \equiv \{ Kernel, RFoob, SVM \}$ \\
~~$\ell_2 MOS_{mdl}^b \leftarrow \alpha_{mdl}^b \frac{\hat{y}_b^{mdl}}{||\hat{y}_b^{mdl}||_2/N_b} $ \\
~~$\hat{y}_b^{B \ell_2 Ensemble} \leftarrow 
      \hat{y}_b^{B \ell_2 Ensemble} + \ell_2 MOS_{mdl}^b$ \\
 \textbf{end} \\
\textbf{output} $\hat{y}_b^{B \ell_2 Ensemble}$ \\
\hline
\end{tabular}
\end{center}
\end{table}
Here, the input of the test data contains $B$-flood terms. Additionally, $\ell_2 MOS_{mdl}^b$ represents the $\ell_2$-normalized model output statistics (MOS) predicted by the base regression model $mdl\in $ $\{ Kernel, RFoob, SVM \}$ on the flood term $b$. Therefore, $\hat{y}_b^{B \ell_2 Ensemble}$ is the ensemble forecast output combined based on three $\ell_2$-normalized model output statistics on the flood term.   

We confirmed that if only when the loss function of ensemble square error is convex, then the accuracy of $\ell_2$-normalized ensemble forecast is strictly greater than the weighted average of the individual base learners accuracy. See the proof of proposition as shown in the supplementary materials. Although the loss function is not always convex, but    $\ell_2$-normalized ensemble forecast has been practically useful as shown in section 3.    

\begin{figure}[h]
\begin{center}
\includegraphics [width = 80mm] {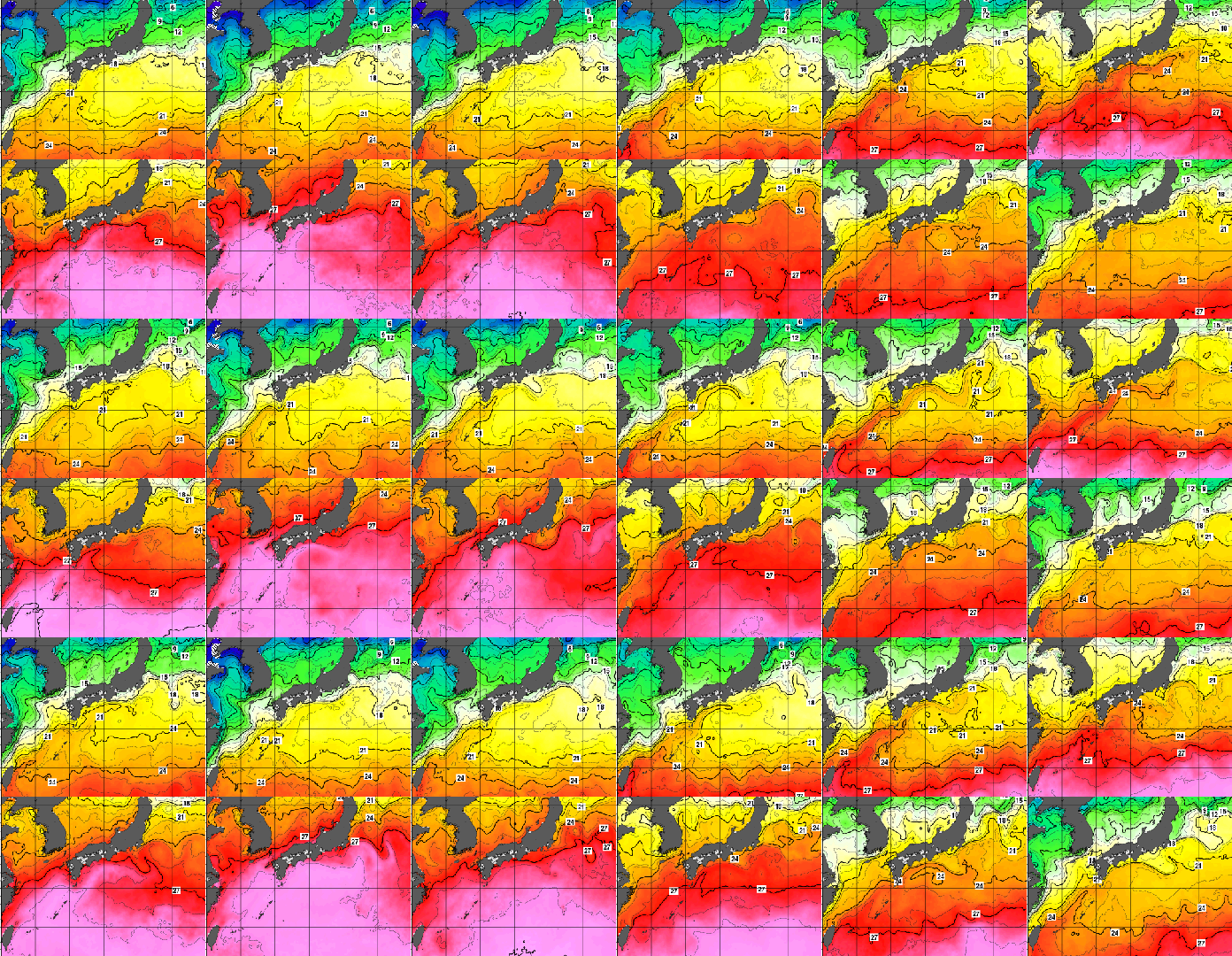}
\end{center}
\caption{Analysed Images of Monthly Average Sea Surface Temperature(from top of left image to right direction, a series of 36 months) }
\label{fig-1}
\end{figure}

\subsection{Weighting Sea Surface Temperature Feature}
We propose the novel weights of target inflow that can create the sea temperature feature vector extracted from SST analysed images. Particularly, we extracted 4,096 dimension vector elements in the fc6 layer of the pre-trained VGG16 network and reduced it to three dimensions of t-SNE. Further, we create the principal component of the “sea temperature weights” using PCA. We found that the weights contributed to the stability of the predictor importance. 

\begin{table}[h]
\begin{center}
\begin{tabular}{l}
\hline 
\textbf{Algorithm2}: Sea Surface Temperature Weights \\
\hline 
\textbf{input} ~$I^{src}_m(m=1,...,M)$ \\
~Trim images $I^{src}_m$ into $I_m$\\
~Select network $Clustering flood$-$month(net^{*})$ \\
~~$net \in \{$ {ResNet101}, {Inception-v3}, {VGG16} $\}$\\
~Feature extract $I_m$ using $net^{*}$ \\ 
~~→feature vectors $z_m=(z_{m1},...,z_{mK})$ \\
~Dimension reduction $z_m$ using {t-SNE} \\
~~→ embedding $v_m=(v_{m1},v_{m2},v_{m3})$ \\ 
~Principal component analysis $v_m$ 90\% variance \\
~~→first component $W_m(m=1,...,M)$ \\
~Standardize $W^{std}_m$ ← $\frac{(W_m - \min W)}{(\max W - \min W)}$ \\
~~here,$\max W = \max\{W_m \}$, $\min W = \min\{W_m \}$ \\
\textbf{output} ~$STW_m$ ← $W^{std}_m + \epsilon$ ~here, $\epsilon=10^{-8}.$ \\
\hline
\end{tabular}
\end{center}
\end{table}

\begin{figure}[h]
\begin{center}
\includegraphics [width =70mm] {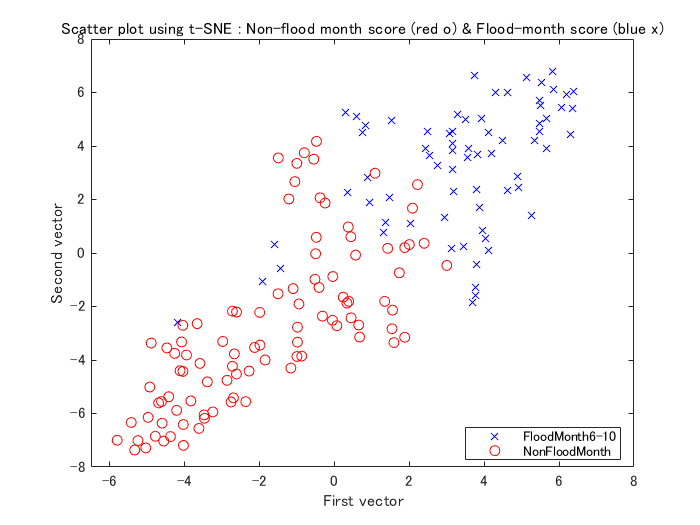}
\end{center}
\caption{Plot {t-SNE} via Inception-v3 from the 312th {avg-pool} layer with 2,048 elements}\label{fig-4}
\begin{center}
\includegraphics [width = 70mm] {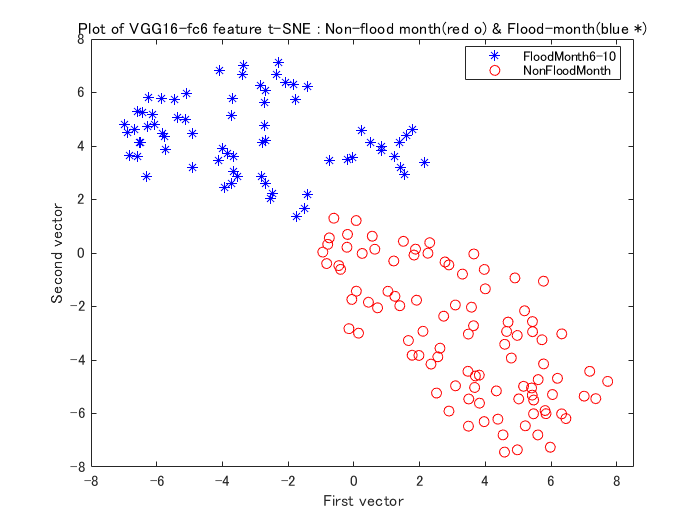}
\end{center}
\caption{Plot {t-SNE} Using VGG16 from the 33th fc6 layer with 4,096 elements}\label{fig-6}
\end{figure}

We propose a novel regression weights approach to balance all predictor importance without decreasing the importance to stable training. The results are presented in Section 3, Applied Results. Using the 10 steps, we reduce the dimensions of an SST analysed image into one dimension for the regression weight to monitor the monthly ocean characteristics. 

1: Collect $I^{src}_m(m=1,...,M):$ sea temperature analysed images per monthly average, from Atmospheric Agency. 
2: Focus on the local sea region where it is possible to identify the dam basin based on rainfall occurrence from meteorological knowledge and experience.
3: Trim $I^{src}_m$ into $I_m$ using image processing.
4: Select pre-trained deep network from ResNet101, Inception-v3, InceptionResNet-v2, VGG16, to recognize weather flood or non-flood terms.
5: Feature extract $I_m$ using the best pre-trained network into feature vectors $z_m=(z_{m1},...,z_{mK})$, e.g. 4,096 elements under VGG16.   
6: Reduce dimensions of $K$-elements feature vectors $z_m$ to embedding vector $v_m=(v_{m1},v_{m2},v_{m3})$ using the t-distributed stochastic neighbour embedding ({t-SNE}, Laurens, 2008), efficient computing Burnes-Hut algorithm(Laurens, 2013).
7: Principal component analysis with 90 percent variance applied in embedding vector $v_m$ to sea temperature weights $W_m(m=1,...,M)$. If it has multiple components, then we can use the first principal component.
8: Min-max standardize, $W^{std}_m = (W_m - minW)/(maxW - minW)$ to avoid negative values because regression weights must be non-negative. Here, $maxW = \max\{W_1,...,W_M \}$ and $minW = \min\{W_1,...,W_M \}$. 
9: Avoid zero weights to add non-zero very small value $STW_m = W^{std}_m + \epsilon$. Here, we can set~$\epsilon=10^{-8}$. 
10: Transform monthly sea temperature weights to hourly date-time expanded weights on flood term range. 

\section{Applied Results}
We apply our method in a dam at Kanto region in Japan, where we focus on the trained term from 2007 to 2018, in the limited flood term June to October. We test the accuracy over the 2019 flood term. 

\subsection{Creation of Sea Surface Feature Weights}
Figure~\ref{fig-6} depicts the plot {t-SNE} of the sea temperature feature extracted by VGG16 from the 33rd fc6 layer with 4,096 elements. We select the VGG16 as the pre-trained network for ocean feature extraction. 
It should be noted that the previous result using the Inception-v3 shown in figure~\ref{fig-4} could not clearly cluster the flood term (Jun to Oct) or non-flood term (Jan to May, Nov, Dec). Particularly, on the centre of the plot, the positions of two terms are very close, thus, the boundary is unclear. 
Similarly, the ResNet101 was partially blending, thus, the subgroup was unclear. Therefore, we used the VGG16 as a superior pre-trained net for the sea temperature feature engineering.       

\subsection{Target and Predictor Data}
Below, $t$ represents the unit interval by one hour, whose period is common to the target and available predictors dataset in this study. We focus on a dam on the flood term from June to October every year. The training term includes 12 years, 2007 to 2018. However, we set the test year 2019 with the flood term June to October, where the 19th typhoon influenced the flood damage of the dam basin. Here, we set the 8-h order for the transformation of auto-regression and moving average using the time series analysis. See the supplementary matearial of our set of hydrological dataset and transformed predictors in detail.    

\subsection{Training Results and Test Accuracy}
Table 4 presents the training results and accuracy comparing the never weighted regressions and weighted regressions. Here, FCD indicates a flood term coefficient of the determinant, evaluated over the flood term. The FCD index makes it possible to fairly compare different flood terms for accuracy, and ranges from zero to one. 
Additionally, RMSE stands for the root mean square error, and MAE is the mean absolute error. These two indices can relatively compare the accuracy under the same flood term. Each flood term has different scales of the dam inflow or basin rainfall, causing an error scale of the forecast outputs.  

Table~\ref{table-4} presents the accuracy comparison of the ocean feature weighted regressions and $\ell_2$-normalized ensemble forecast for the 6-h lead time dam inflow. Our proposed sea temperature feature weighted regression almost outperforms the never weighted model at both the base learner and $\ell_2$-normalized ensemble forecast.    
Here, the common $\ell_2$-norm is applied in all the test flood terms. We confirm whether the batch term $\ell_2$-normalized ensemble can improve the accuracy using our proposed algorithm 2. 
Table~\ref{table-5} is well-informed about a computed $\ell_2$-norm at 5-flood term for ensemble forecasts, here $b$ stands for a batch of flood term. Note that the first and second flood term has low value of $\ell_2$ norm, so the peak levels of dam inflow are relatively small. In contrast, fourth and fifth flood term has larger value of $\ell_2$ norm, so the peak levels of dam inflow are extremely high. Thus, base model's ($mdl$) coeficient at flood batch $\alpha^b_{mdl}$ are not able to equalize them. We implemented an ensemble forecast that named $\ell_2$-normalized ensemble using median$+\sigma$, each base model has common coeficient that consists the median and standard deviation of 5-flood batch's $\ell_2$-norm. This accuarcy outperformed rather than the ensemble forecast of equalized coeficients with one third. 
Furthermore, our final ensemble forecast that named $+$flood batch of $\ell_2$-norm was implemented using the most detailed base model's ($mdl$) coeficient at flood batch $\alpha^b_{mdl}$. The final accuarcy of ensemble forecast outperformed rather than the ensemble forecast of median standard deviation coeficients.        

\begin{table}[h]
\caption{Experimental Results of $\ell_2$-norm at 5-flood term for Ensemble Forecasts ($b$ stands for flood term)} \label{table-5}
\begin{center}
\begin{tabular}{l | r r r r r}
\textbf{MODEL} &\textbf{b=1} &\textbf{b=2}  &\textbf{b=3}  &\textbf{b=4} &\textbf{b=5} \\
\hline 
Kernel Reg.  & 1.01& 0.86& 1.22& 3.23& 9.72 \\
\hline 
RForest oob.& 1.93& 1.24& 3.17& 8.15& 25.18\\
\hline 
SVM poly    & 1.09& 0.79& 1.72& 3.82& 28.25\\
\hline 
\end{tabular}
\end{center}
\end{table}

%%%
\begin{table}[h]
\caption{Accuracy of Ocean Weighted Regressions and the Outputs $\ell_2$-normalized Ensemble Forecast for 6-h leadtime Dam Inflow} \label{table-4}
\begin{center}
\begin{tabular}{l c | r r r}
\textbf{MODEL} &\textbf{WEIGHT} &\textbf{FCD}  &\textbf{RMSE}  &\textbf{MAE} \\
\hline 
Kernel   & never  & 0.3702  & 120.97  & 33.04  \\
Reg.      & \textbf{Ws on} & \textbf{0.4865}  & \textbf{109.23}  & \textbf{30.27}  \\
\hline 
RForest  & never  &  0.7895  & 69.93   & 20.99  \\
oob.        & Ws on &  0.7789  & 71.67   & 21.48  \\
\hline 
SVM     & never  &  0.8900  &  50.55   & 18.47  \\
poly.      & \textbf{Ws on} &  \textbf{0.9049}  &  \textbf{47.01}  & \textbf{18.26} \\
\hline 
$\ell_2$-norm & \textbf{never}  &  \textbf{0.8930}  &  \textbf{48.85}   & \textbf{18.42}  \\
Ensemble& Ws on &  0.8837  &  51.97  & 18.33 \\
\hline 
$\ell_2$-norm & \textbf{never}  &  \textbf{0.9015}  &  \textbf{47.84}   & \textbf{16.43}  \\
median$+\sigma$ & Ws on &  0.8876  &  51.09  & 17.77 \\
\hline 
$+$flood $b$ & \textbf{never}  &  \textbf{0.9131}  &  \textbf{44.92}   & \textbf{15.22}  \\
$\ell_2$-norm& \textbf{Ws on} &  \textbf{0.9035}  &  \textbf{47.36}  & \textbf{16.29} \\
\hline 
all        & never  &  0.7929  & 63.84    & 20.42  \\
           & \textbf{Ws on} &  \textbf{0.8075}  &  \textbf{63.05}  & \textbf{20.40} \\
\hline 
\end{tabular}
\end{center}
\end{table}

\begin{figure}[h]
\begin{center}
\includegraphics [width = 100mm] {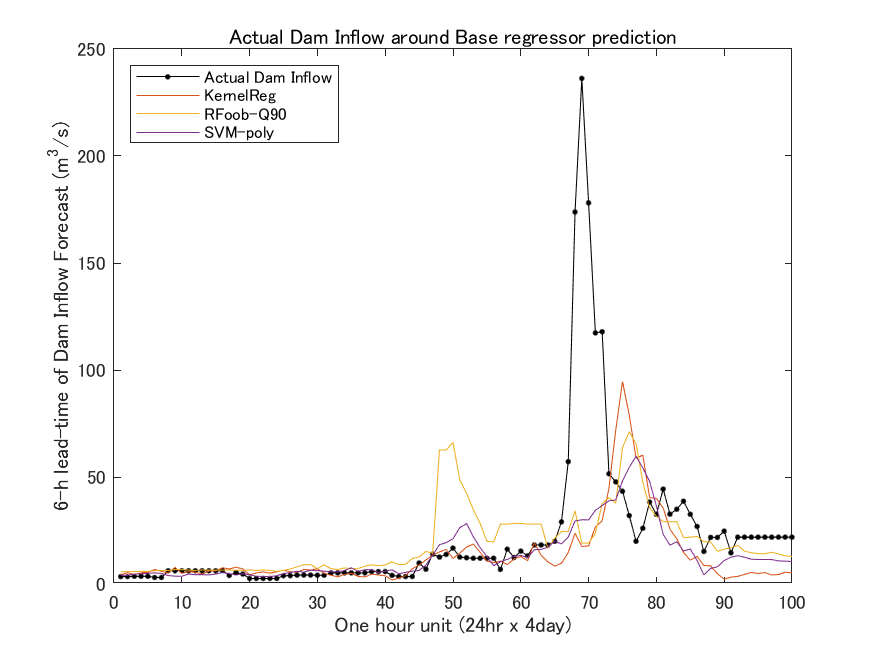}
\end{center}
\caption{Dam Inflow 6-h Forecasts and Actual Observation -Multiple Peak Forecast: peak level 230 ton/s in August 2019}\label{fig-13}
\begin{center}
\includegraphics [width =100mm] {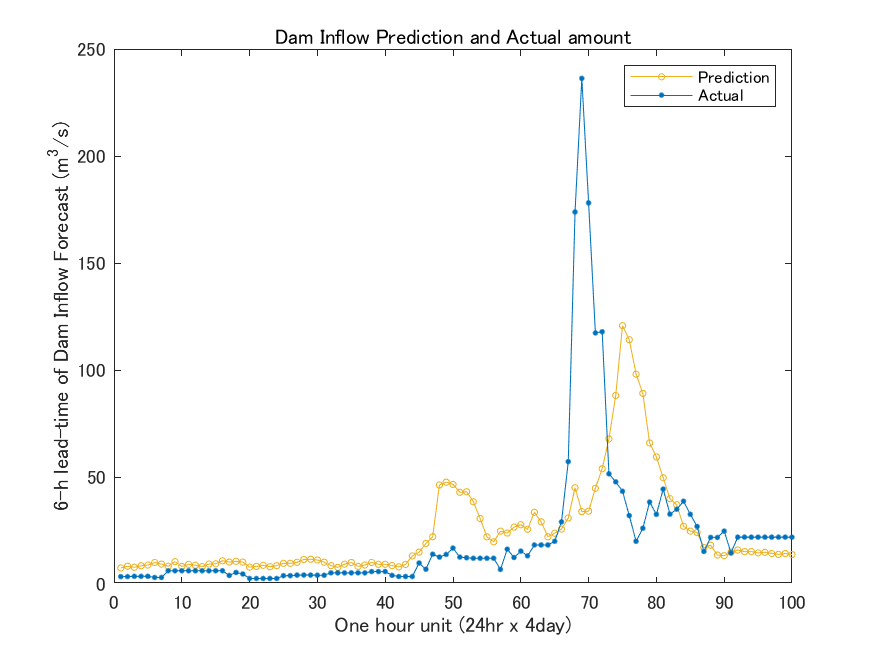}
\end{center}
\caption{Ensemble Dam Inflow Forecast and Actual Observation -Multiple Peak Forecast: August 2019}\label{fig-18}
\end{figure}

\subsection{Base Regression Model Outputs}
Figure~\ref{fig-13} shows a heavy rain's flood of the dam inflow 6-h lead time forecast and actual observation with the peak level of 230 ton/s in August 2019. The base learner's prediction output delays than the actual peak time. The RFoob is over-fitted around the date-time 50. 
Figure~\ref{fig-15} shows an anomalous and enexperienced flood inflow outside range of the training data, the dam inflow 6-h lead time forecast and actual observation with the peak level of 1,750 ton/s in October 2019. The SVM base learner's prediction output makes it possible to almost match the peak time, effectively approximating the shape of the actual hydrograph. 

\subsection{Model Outputs $\ell_2$-Normalized Ensemble Forecasts} %%
Figure~\ref{fig-18} shows the $\ell_2$-normalized ensemble dam inflow forecast and actual observation at a heavy rain's flood inflow. This ensemble forecast automatically selects the best learner's output, combined with the $\ell_2$-normalized denominator to reduce the base learner's variance. 
Figure~\ref{fig-20} shows the $\ell_2$-normalized ensemble dam inflow forecast and actual observation at an anomalous dam inflow outside range of the training data. The ensemble forecast is automated to combine three base learner's outputs using the $\ell_2$-normalized denominator to reduce the base learner's variance, successfully. 

\begin{figure}[h]
\begin{center}
\includegraphics [width = 100mm] {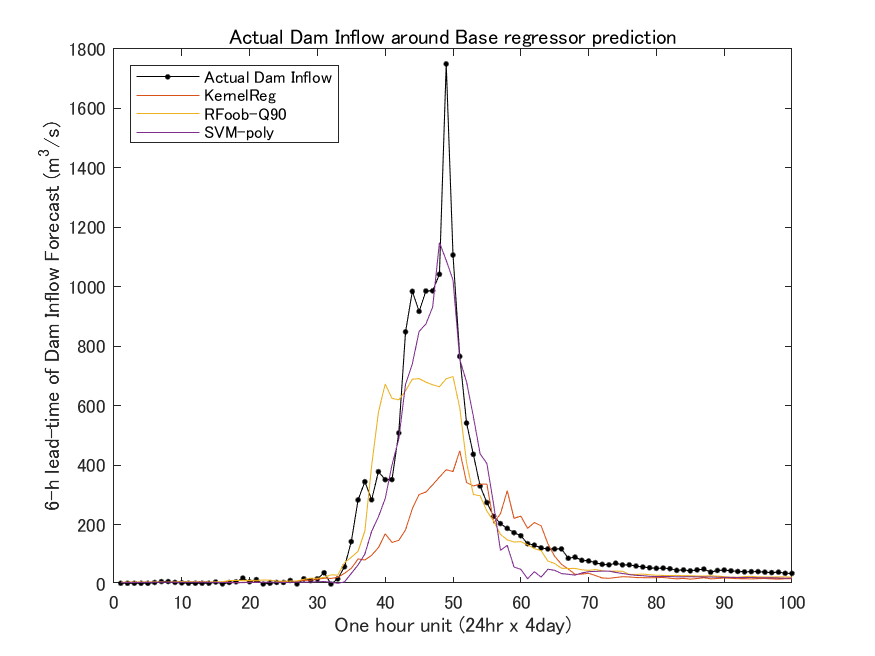}
\end{center}
\caption{Dam Inflow 6-h Forecasts and Actual Observation -Anomalous Flood Forecast: peak level 1,750 ton/s in October 2019}\label{fig-15}
\begin{center}
\includegraphics [width = 100mm] {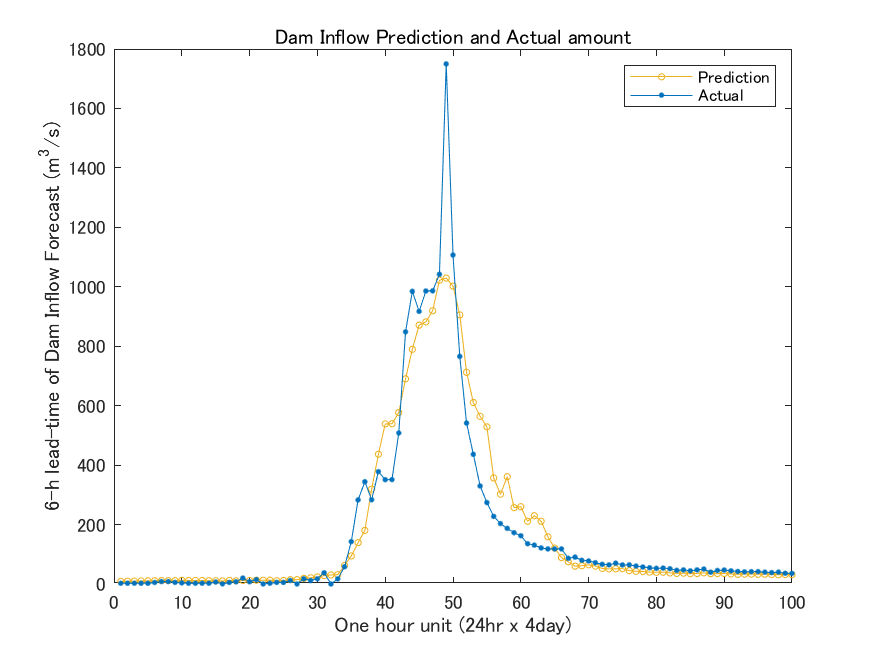}
\end{center}
\caption{Ensemble Dam Inflow Forecast and Actual Observation -Anomalous Flood Forecast: October 2019}\label{fig-20}
\end{figure}

\subsection{Limitation to Complex and Anomalous Flood Forecasts} %%
Our proposed method has some restriction of data mining and fine tuning ensemble model. 
Flood data mining is possible to collect uncertain number of floods every year. And basin rainfall observations are chaotic nature that contains measurement error and unknown noise. The predictive accuracy is not always perfect, for example, there is a case that a computed ensemble forecast is over-fitted at the peak prediction, and that another forecast has start time of gap between a forecast time and the actual one. Unlike a single target flood inflow modeling, we think that a key is the dependence structure with another correlated target variable such as basin rainfall. The rainfall feature is complex and chaotic property, for example, there are multiple peak with different levels, and long time rainfall could occurre at a same region that brought an unexperienced scale of dam inflow.     
We continue to collect signal included data and statistically learn ensemble forecast model. 

\subsection{Predictor Importance Weighting Sea Surface Temperature}
As shown in Figure~\ref{fig-7}, the vertical axis indicates the predictor importance comparison without and with the sea temperature weights computed on the random forest regression minimizing the out-of-bag. The horizontal axis represents all the predictors that are grouped in the case of the never weighted regression. The predictor importance of the never weighted regression gradually decreases to close to the zero value. In contrast, the predictor importance of the sea temperature feature weighted regression maintains a higher value than the importance value of the never weighted regression. This means that the higher value and balance of the predictor importance facilitates stable training.    
See the supplementary material that highlights the predictor importance of the input data and transformed feature. 

\section{Concluding Remarks}
Finally, the following conclusions are drawn from our applied results and further apps improvement.

\subsection{Stable Weighted Regressions and Outperformed Ensemble}
Our proposed weighted regression method using the ocean feature contributed to the stability of the predictor importance index and test accuracy applied on the flood term of  dam inflow forecasts.  
To extract the analysed images of the sea temperature, the VGG16 network with the fc6 layer is superior to other pre-trained networks.    
We found that our ocean weighted regression method contributed to the stability of the predictor importance for more effective calibration. Furthermore, the model output $\ell_2$-normalized ensemble outperformed each single weighted regression by numerical experiments and the proof of proposition, see the supplementary material.

\subsection{Future Works for Statistical Learning Unknown Flood}
This study addressed the forecast problem of the variability of dam inflow using 6-h lead time to forecast future occurrence of extreme rainfall with uncertain start times and peak levels. 
Single target forecast has limited capability for the representation of extreme dam inflow and the interpretation of multiple peaks with the time lag in extreme rainfall. Further, we address the multiple target forecast problem incorporating the dependence of basin rainfall and dam inflow for better transparency of meteorological and hydrological phenomena.   

\subsubsection*{Acknowledgements}
%We gratefully acknowledge the help provided by constructive comments of the anonymous reviewers. 
We thank T. Fukumoto and S. Kuramoto (MathWorks) for their support to set the pace using the MATLAB.

\clearpage % -------------------------------------------

\begin{figure}[h]
\begin{center}
\includegraphics [width = 120mm] {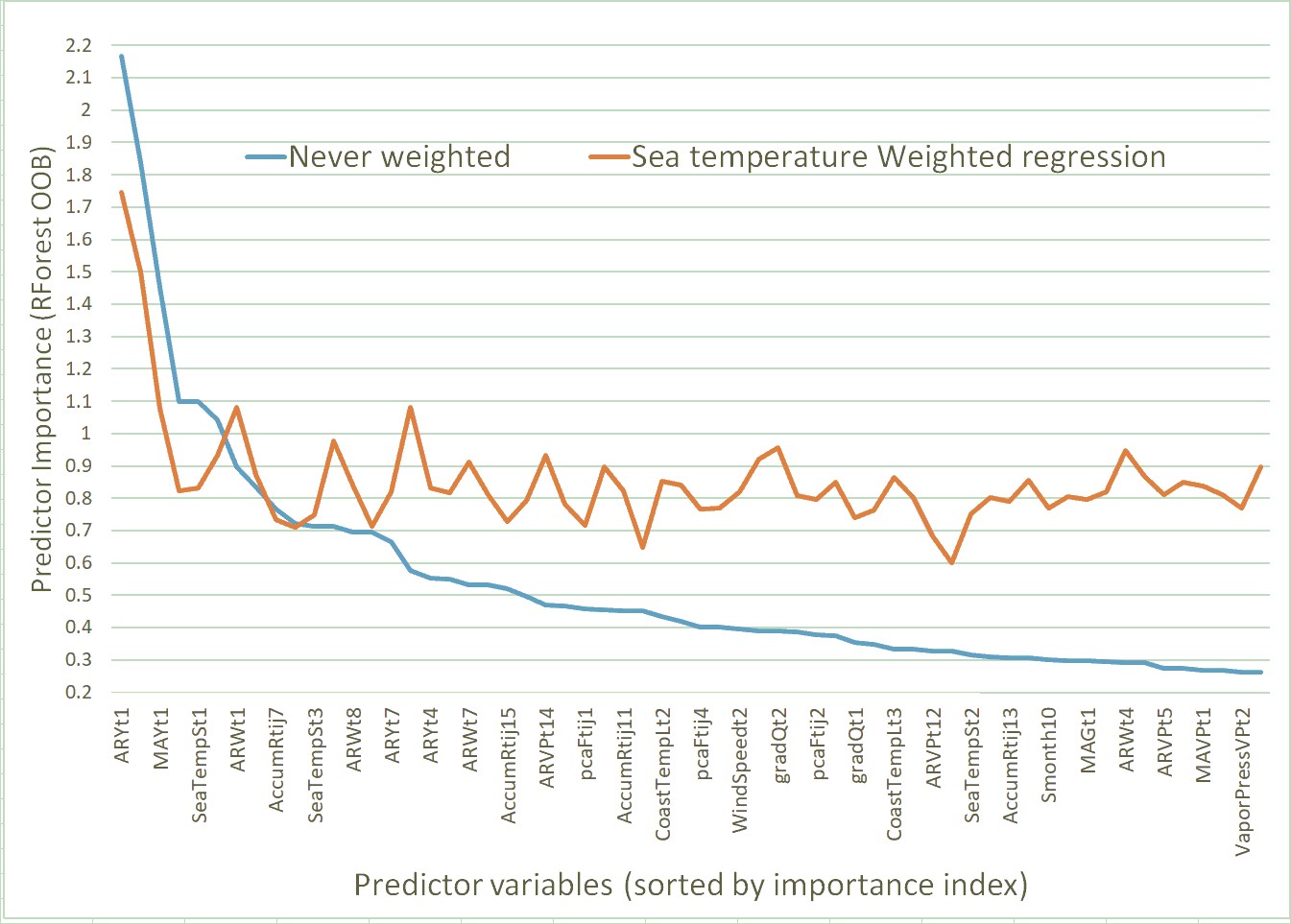}
\end{center}
\caption{Predictor Importance Comparison Without/With Sea Surface Temperature Weights computed on Random Forest out-of-bag regression}\label{fig-7}
\end{figure}

\section*{Reference}

\bibliographystyle{splncs04}
\bibliography{egbib}

[1] Balmaseda, M. A., Alves, O. J., Arribas, A., et al. (2009).
    Ocean Initialization for Seasonal Forecasts..
   \textit{Oceanography 22.}: 154--159.

[2] Beggs, H. (2010).
    Use of TIR from Space in Operational Systems.
   \textit{Oceanography from Space Revisited.}: edited by V. Barale, J. F. R. Gower.

[3] Breiman, L. (2001).
     Statistical Modeling: The Two Cultures.
    \textit{Statistical Science Vol.16, No.3}: 199--231.

[4] Breiman, L. (2001).
     Random Forests.
    \textit{Machine Learning, 45}: 5--32. 

[5] Bremness, J.B. (2004) .
     Probabilistic Forecasts of Precipitation in Terms of Quantiles Using NWP Model Output.
    \textit{Monthly Weather Reiew. 132}: 338--347.

[6] Bremness, J.B. (2020).
     Ensemble Postprocessing Using Quantile Function Regression Based on Neural Networks and Bernstein Polynomials.
    \textit{Monthly Weather Reiew. 148}: 403--414.

[7] Brunsdon C, Fotheringham A.S. and Charlton, M (1998).
    Geographically Weighted Regresson-Modeling Spatial Non-stationarity.
   \textit{The Statistician 47}: 431--443.

[8] Cleveland, W., Develin, S. (1988).
   Locally Weighted Regression: An Approach to Regression Analysis by Local Fitting. 
  \textit{Journal of the American Statistical Association. 83}: 590--610.

[9] Delle Monache, L., T. Eckel, et al., (2013).
     Probabilistic Weather Prediction with An Analog Ensemble.
    \textit{Monthly Weather Reiew. 141}: 3498--3516.

[10] Fan, R.-E., P.-H. Chen, and C.-J. Lin. (2005).
     Working Set Selection Using Second Order Information for Training Support Vector Machines.
    \textit{Journal of Machine Learning Research, Vol 6}: 1889--1918.

[11] Foody, G. M, (2003).
    Geographical Weighting as A Further Refinement to Regression Modelling: An Example Focused on the NDVI–Precipitation Relationship.
  \textit{Remote Sensing of Environment, 88}: 283--293.

[12] Gneiting, T. A.E. Raftery, et al. (2005).
     Calibrated Probabilistic Forecasting Using Ensemble Model Output Statistics and Minimum CRPS Estimation.
    \textit{Monthly Weather Reiew. 133}: 1098--1118.

[13] Hemri, S. (2018).
     Applications of Porstprocessing for Hydrological Forecasts. 
    \textit{Statistical Postprocessing of Ensemble Forecasts. chapter 8}: edited by S. Vannitsem, D.S. Wilks, J.W. Messner, Elsevier. 

[14] Hwang, J.,P. Orenstein., J. Cohen, K. Pfeiffer, L. Mackey. (2019).
     Improving Subseasonal Forecasting in the Western U.S. with Machine Learning.
    \textit{ACM, KDD2019}.

[15] Huang, P. S., H. Avron, T. N. Sainath, V. Sindhwani, and B. Ramabhadran. (2014).
     Kernel methods match Deep Neural Networks on TIMIT.
    \textit{IEEE International Conference on Acoustics, Speech and Signal Processing.}: 205--209.

[16] Japan Meteorological Agency.
    Ocean Data Assimilation System (MOVE, MRI.COM-G2).
  \textit{Tokyo Climate Center WMO Regional Climate Center in RA II.}

[17] Kecman V., T. -M. Huang, and M. Vogt. (2005).
     Iterative Single Data Algorithm for Training Kernel Machines from Huge Data Sets: Theory and Performance.
    \textit{In Support Vector Machines: Theory and Applications.}: Edited by Lipo Wang, 255-–274, Springer-Verlag.

[18] Kobayashi, S. et al. (2015).
    The JRA-55 reanalysis: General Specifications and Basic Characteristics.
  \textit{Journal of Meteorological Society Japan, 93}.

[19] Lang, M.N., G.J. Mayr, et al., (2019)
     Bivariate Gaussian models for Wind Vectors in a Distributional Regression Framework.
    \textit{Advances in Statistical Climatology, Meteorology and Oceanography. 5}: 115--132.

[20] Le, Q., T. Sarlós, and A. Smola. (2013).
     Fastfood — Approximating Kernel Expansions in Loglinear Time.
    \textit{Proceedings of the 30th International Conference on Machine Learning, Vol. 28, No. 3}: 244--252.

[21] Messner, J.W., G.J. Mayr, et al., (2017).
     Nonhomogeneous Boosting for Predictor Selection in Ensemble Postprocessing.
    \textit{Monthly Weather Reiew. 145}: 137--147.

[22] Ministry of Land Infrastructure. (2019).
   Transport and Tourism : Basic Policy Strengthening Flood Control Function of Existing Dam.

[23] Moller, A. L. Spazzini, et al., (2018).
     Vine Copula Based Post-processing of Ensemble Forecasts for Temperature.
    \textit{arXiv:1811.02255}.

[24] Nicolai Meinshausen. (2006).
     Quantile Regression Forests.
    \textit{Journal of Machine Learning Research, 7}: 983--999. 

[25] O'Carroll, A.G., E.M. Armstrong, H.M. Beggs, et al. (2019).
    Observational Needs of Sea Surface Temperature.
   \textit{Frontiers in Marine Science Vol.6, Article 420}: 1--27.

[26] Pinson, P., J.W., Messner (2018).
     Application of Postprocessing for Renewable Energy. 
    \textit{Statistical Postprocessing of Ensemble Forecasts. chapter 9}: edited by S. Vannitsem, D.S. Wilks, J.W. Messner, Elsevier. 

[27] Propastin, P., M. Kappas. (2008).
   Reducing Uncertainty in Modelling NDVI Precipitation Relationship: A
Comparative Study Using Global and Local Regression Techniques.
  \textit{GIScience and Remote Sensing, 45}: 1--25.

[28] Raftery, A.E,~Gneiting, T. et al. (2005).
     Using Bayesian Model Averaging to Calibrate Forecast Enembles.
    \textit{Monthly Weather Reiew. 133}: 1155--1174.

[29] Rahimi, A., and B. Recht. (2008).
     Random Features for Large-Scale Kernel Machines.
    \textit{Advances in Neural Information Processing Systems Vol.20}: 1177--1184.

[30] Schlosser, L., T. Hothorn, et al., (2019).
     Distributional Regression Forests for Probabilistic Precipitation Forecasting in Complex Terrain.
    \textit{Annals of Applied Statistics 13}: 1564--1589.

[31] Scholkopf, B., A.J. Smola, Williamson, and P.L. Bartlett. (2000).
     New Support Vector Algorithms.
    \textit{Neural Computation}: 1207-1245.

[32] Scholkopf, B.,~A.J.~Smola. (2002).
     Regression Estimation.
    \textit{Learning with Kernels:~Support Vector Machines,~Regularization,~Optimization, and Beyond., 9}: 251--277.

[33]  Taillardat, M.,~O. Mestre, et al.~(2016).
     Calibrated Ensemble Forecasts Using Quantile Regression Forests and Ensemble Model Output Statistics.
    \textit{Monthly Weather Reiew. 144}: 2375--2393.

[34] Trevor Hastie, Robert Tibshirani, Jerome Friedman. (2009).
     Random Forest.
    \textit{The Elements of Statistical Learning:~Data Mining, Inference, and Prediction, 15}: 587--604, Springer-Verlag. 

[35] Tsujino, H., T. Motoi, I. Ishikawa, et al. (2010).
    Reference Manual for the Meteorological Research Institute Community Ocean Model (MRI.COM) version 3.
  \textit{Technical Reports of the Meteorological Research Institute, 59}: 1--273.

[36] van der Maaten, Laurens, and Geoffrey Hinton (2008).
     Visualizing Data using t-SNE.
    \textit{Journal of Machine Learning Research 9}: 2579--2605.

[37] van der Maaten, Laurens (2013).
     Burnes-Hut t-SNE.
    \textit{arXiv:1301.3342}.

[38] Vannitsem,S. ,J.B. Bremnes, J. Demaeyer, G.R. Evans, J. Flowerdew, et al.(2020).
     Statistical Postprocessing for Weather Forecasts -Review, Challenges and Avenues in a Big Data World.
    \textit{arXiv2004.06582.}

[39] Van Schaeybroeck, B. and S. Vannitsem (2015).
     Ensemble Post-processing Using Member-by-member Approaches: Theoretical Aspects.
    \textit{Qurterly Journal of Royal Meteorological Society. 141}: 807--818.

[40] Wilks, D.S. (2018).
     Univariate Ensemble Postprocessing. 
    \textit{Statistical Postprocessing of Ensemble Forecasts. chapter 3}: edited by S. Vannitsem, D.S. Wilks, J.W. Messner, Elsevier. 

\clearpage

%---------------- Supplementary Materials ----------------------
\section{Supplementary Materials}

\subsection{Proof of Proposition}
As we proposed in the subsection 2.1, the algorithm 1 represents the flood batch-term $\ell_2$-normalized ensemble that consists each outputs of base learners. Below is a supplementary proof of proposition that $\ell_2$-normalized ensemble has better skill, that is higher accuracy, rather than any individual base learner.
\subsection{Better Skill Proposition of $\ell_2$-normalized Ensemble Algorithm1} %%%
Now we represent $\textbf{y}=(y_1,...,y_N)$ as a target vector of actual values with the number of samples $N$. On the other hand, $\hat{ \textbf{y} }^{mdl}=(\hat{y}_1^{mdl},...,\hat{y}_N^{mdl})$ means an individual prediction outputs using a base trained models $mdl \in \{1,...,M\}$. 
A $\ell_2$-normalized ensemble forecast is represented as a summation of base model outputs.
\begin{equation}
\hat{\textbf{y}}^{\ell_2{-normalized Forecast}} = 
\sum_{mdl=1}^{M} a_{mdl} \frac{\hat{\textbf{y}}^{mdl}}{||\hat{\textbf{y}}^{mdl}||_2/N} \label{eq-1} 
\end{equation}
Next equation(\ref{eq-2}) defines the accuracy of mean square error $MSE(\hat{\textbf{y}},\textbf{y})$, and equation(\ref{eq-3}) defines cosine similarity $sim(\hat{\textbf{y}},\textbf{y})$, called "skill" in meteorological and hydrological science. 
Note that $\inp{ \hat{ \textbf{y} } } { \textbf{y} } $ is the inner product.
\begin{equation}
MSE(\hat{\textbf{y}},\textbf{y}) = \frac{1}{N}\sum_{t=1}^{N} ( \hat{y}_t - y_t )^2 \label{eq-2}
\end{equation}
\begin{equation}
sim(\hat{\textbf{y}}, \textbf{y}) = 
\frac{ \inp{ \hat{ \textbf{y} } } {\textbf{y} } }
{ || \hat{ \textbf{y} } ||_2 || \textbf{y} ||_2 } \label{eq-3}
\end{equation}
We are able to derive a relationship between equation(\ref{eq-2}) and equation(\ref{eq-3}), because $N, || \hat{\textbf{y}} ||_2, || \textbf{y} ||_2$ are positive.
\begin{equation}
\frac{MSE(\hat{\textbf{y}},\textbf{y})} { || \hat{ \textbf{y} } ||_2 || \textbf{y} ||_2 } =
\frac{|| \hat{ \textbf{y} } ||_2} {N || \textbf{y} ||_2} - \frac{2}{N}sim(\hat{\textbf{y}}, \textbf{y})
+ \frac{|| \textbf{y} ||_2} {N || \hat{\textbf{y}} ||_2} \label{eq-4}
\end{equation}
Thus, the accuracy of mean square error $MSE(\hat{\textbf{y}},\textbf{y})$ is in inverese proportion to the skill $sim(\hat{\textbf{y}},\textbf{y})$. 
% Better Accuracy Proposition
% Proposition
% Pre-set \usepackage{amsthm} % for Proposition 

\begin{prop}
The skill of $\hat{\textbf{y}}^{\ell_2{-normalized Forecast}}$ is strictly better than the weighted skill of base model's outputs $\hat{\textbf{y}}^{mdl}(mdl=1,...,M)$.
For any vector of weights $\textbf{a}=\{a_{1},...,a_M \} \in \textbf{R}^M$ with $\sum_{mdl=1}^M a_{mdl}=1$, and $a_{mdl} \geq 0$. 
Let the $N$ multipied $\hat{\textbf{y}}^{\ell_2{-normalized Forecast}}$ in equation(\ref{eq-1}) to evaluate the skill of base outputs as next measure 
\begin{equation}
\bar{\textbf{y}}_{(\textbf{a})} \equiv \sum_{mdl=1}^{M} a_{mdl} \frac{\hat{\textbf{y}}^{mdl}}{||\hat{\textbf{y}}^{mdl}||_2} \label{eq-5} 
\end{equation}
Then a next better skill of inequality holds strictly whenever it holds the non-zero condition that $\sum_{mdl=1}^{M} a_{mdl} \cdot sim(\hat{\textbf{y}}^{mdl}, \textbf{y})\neq 0$. 
\begin{equation}
\sum_{mdl=1}^{M} a_{mdl} \cdot sim(\hat{\textbf{y}}^{mdl}, \textbf{y}) \leq 
sim(\bar{\textbf{y}}_{(\textbf{a})}, \textbf{y}), \label{eq-6} 
\end{equation}
Hence, whenever the $\ell_2$-normilized weights of individual base model's skill is positive, the skill of $\bar{\textbf{y}}_{(\textbf{a})}$ is strictly greater than the weighted average of the individual base learenr's skills.
\end{prop}

%\clearpage % 改ページ

% Proof setting below
% In front of document 
% Pre-set \usepackage{mathtools}
% Pre-define \DeclarePairedDelimiterX{\inp}[2]{\langle}{\rangle}{#1, #2} for inner product
% Pre-set \usepackage{amsthm} % for Proof QED, 

\begin{proof}
\begin{equation}
\sum_{mdl=1}^{M} a_{mdl} \cdot sim(\hat{\textbf{y}}^{mdl}, \textbf{y}) =
\sum_{mdl=1}^{M} a_{mdl} \inp{ \frac{\hat{\textbf{y}}^{mdl}}{||{\textbf{y}}^{mdl}||_2} }
{ \frac{{\textbf{y}}}{||\hat{\textbf{y}}||_2} }  \nonumber
\end{equation}
\begin{equation}
= \inp{ \sum_{mdl=1}^{M} a_{mdl} \frac{{\textbf{y}}^{mdl}}{||{\textbf{y}}^{mdl}||_2} }
{ \frac{{\textbf{y}}}{||{\textbf{y}}||_2} } 
= \inp{ \bar{\textbf{y}}_{(\textbf{a})} } { \frac{{\textbf{y}}}{||{\textbf{y}}||_2} }  
 \nonumber
\end{equation}
\begin{equation}
= \inp{ \frac{ \bar{\textbf{y}}_{(\textbf{a})} }{ ||\bar{\textbf{y}}_{(\textbf{a})}||_2 } } { \frac{{\textbf{y}}}{||{\textbf{y}}||_2} }
|| \bar{\textbf{y}}_{(\textbf{a})} ||_2
= sim(\bar{\textbf{y}}_{(\textbf{a})}, \textbf{y}) || \bar{\textbf{y}}_{(\textbf{a})} ||_2 
\label{eq-8}
\end{equation}
Scince the base models forecasts are distinct, Jensen's inequality now yields the magnitude claim as 
\begin{equation}
|| \bar{\textbf{y}}_{(\textbf{a})} ||_2 = 
\sum_{mdl=1}^{M} a_{mdl} \frac{\hat{\textbf{y}}^{mdl}}{||\hat{\textbf{y}}^{mdl}||_2}  \nonumber
\end{equation}
\begin{equation}
\leq
\sum_{mdl=1}^{M} a_{mdl} \frac{||\hat{\textbf{y}}^{mdl}||_2}{||\hat{\textbf{y}}^{mdl}||_2} = 1
 \label{eq-9}
\end{equation}
We derived next inequality from equations (\ref{eq-8})(\ref{eq-9})
with strict inequality when $\sum_{mdl=1}^{M} a_{mdl} \cdot sim(\hat{\textbf{y}}^{mdl}. \textbf{y})\neq 0$.
\begin{equation}
\sum_{mdl=1}^{M} a_{mdl} \cdot sim(\hat{\textbf{y}}^{mdl}, \textbf{y}) \leq 
sim(\bar{\textbf{y}}_{(\textbf{a})}, \textbf{y}) \label{eq-10}
\end{equation}
This is the core of proposition that states the right hand side of equation(\ref{eq-10}), that is $\ell_2$-norm ensemble forecast has greater skill of cosine similarity rather than the left side that is a weighted summation of individual base learners skills.
\end{proof}

\subsection{Ocean Feature Plot}

\subsubsection{Sea Surface Temperature t-SNE Feature Extracted from Pre-tarined Nets} 
As we proposed in the subsection 2.2, the algorithm 2 represents how to extract sea surface temperature weights from the analysed images. Below is supplementary experiments that demensionality reduction by t-SNE extracted from another pre-tarined networks.

\begin{figure}[h]
\begin{center}
\includegraphics [width = 60mm] {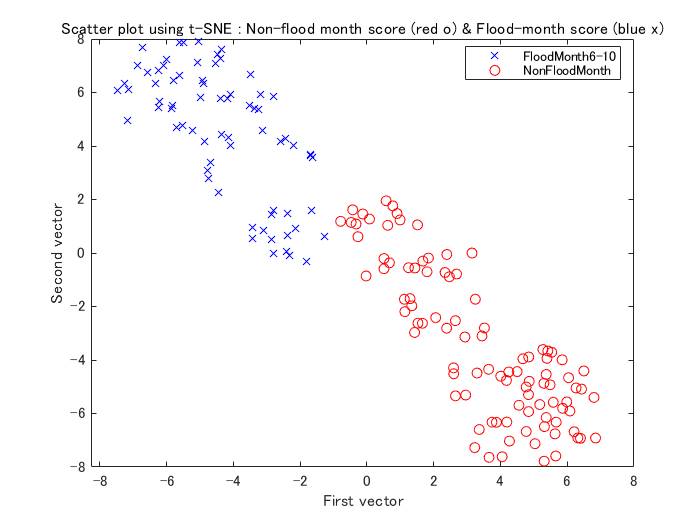}
\end{center}
\caption{Plot {t-SNE} Using ResNet101 from the 344th {pool5} layer with 2,048 elements}\label{fig-3a}
\begin{center}
\includegraphics [width = 60mm] {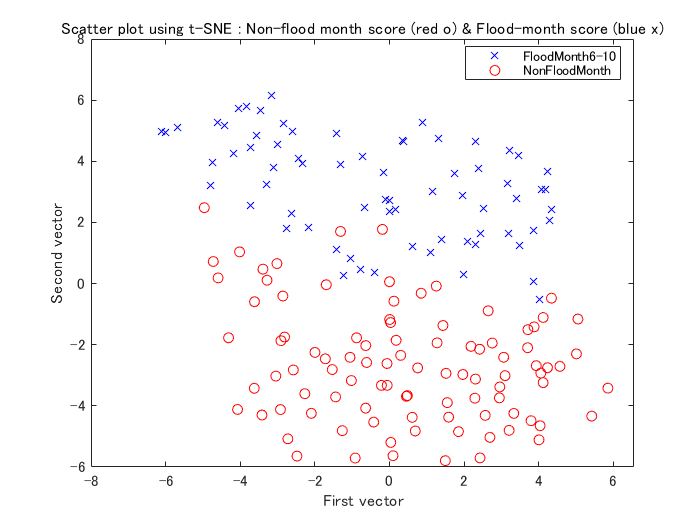}
\end{center}
\caption{Plot {t-SNE} via InceptionResNet-v2 from the 821st {avg-pool} layer with 1,536.}\label{fig-3b}
\end{figure}

%\clearpage % 改ページ

\subsection{Predictors List and the Importance}
As we shown in the subsection 3.7, ocean weighted regression are applied and dipicted the whole predictor importance compared between never weighted case and ocean weighted case in Figure 10. Below is the explanation in detail subdivided into three group such as hydrological, meteorological, and rainfall.

\subsubsection{Predictor Importance Random Forest Minimising Out-of-bag Error}
\paragraph{Predictor Dataset and Statistical Transformations} %%%
As we implemented in the section 3, $\ell_2$-normalized ensemble forecast and sea ocean weights are applied and the results are shown in the numerical experiments.
As shown in Table~\ref{table-M1}, Table~\ref{table-M2}, Table~\ref{table-M3}, 
we show the predictor data list in detail, i.e. the hydrological, meteorological, and rainfall respectively.

\begin{figure}[h]
\begin{center}
\includegraphics [width = 65mm] {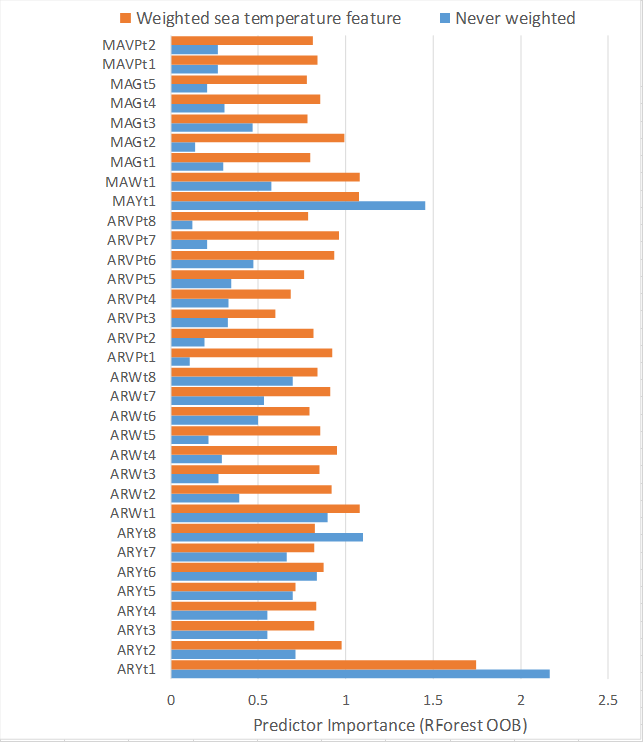}
\end{center}
\caption{Predictor Importance of Hydrological Data and Auto-regressive and Moving Averaged Data}\label{fig-9}
\end{figure}

\begin{table}[h]
\caption{Set of Hydrological Data and ARMA Transformed Predictors} \label{table-M1}
\begin{center}
\begin{tabular}{lp{120mm} lp{120mm}}
\textbf{NAME}  &\textbf{DESCRIPTION} \\
\hline
ARY  &$AR(p)$-$Y_t$: Auto-regressive (AR) transformed dam inflow with $p$-order, $(Y_{t-j}, j=1,...,p)$ \\
ARW  &$AR(p)$-$W_t$: AR transformed river water height with $p$-order, $(W_{t-j}, j=1,...,p)$ \\
ARVP  &$AR(p)$-$VP_t$: AR transformed vaporous pressure with $p$-order, $(VP_{t-j}, j=1,...,p)$ \\
MAY  &$MA(p)$-$Y_t$: Moving-average (MA) transformed dam inflow with $p$-order, $MA(p)$-$Y_t$ $=$ $\sum_{j=1}^p Y_{t-j}/p$ \\
MAW  &$MA(p)$-$W_t$: MA transformed river water height with $p$-order, $MA(p)$-$W_t$ $=$ $\sum_{j=1}^p W_{t-j}/p$ \\
MAG  &$MA(p)$-$G^r_t$: MA transformed ground observed rainfall at multiple basin positions $(r=1,...,5)$ with $p$-order, $MA(p)$-$G^r_t$ $=$ $\sum_{j=1}^p G^r_{t-j}/p$ \\
MAVP &$MA(p)$-$VP^m_t$: MA transformed vaporous pressure with $p$-order, $MA(p)$-$VP^m_t$ $=$ $\sum_{j=1}^p VP^m_{t-j}/p$ $(m=1,2)$ \\
\hline
\end{tabular}
\end{center}
\end{table}

\begin{figure}[h]
\begin{center}
\includegraphics [width = 70mm] {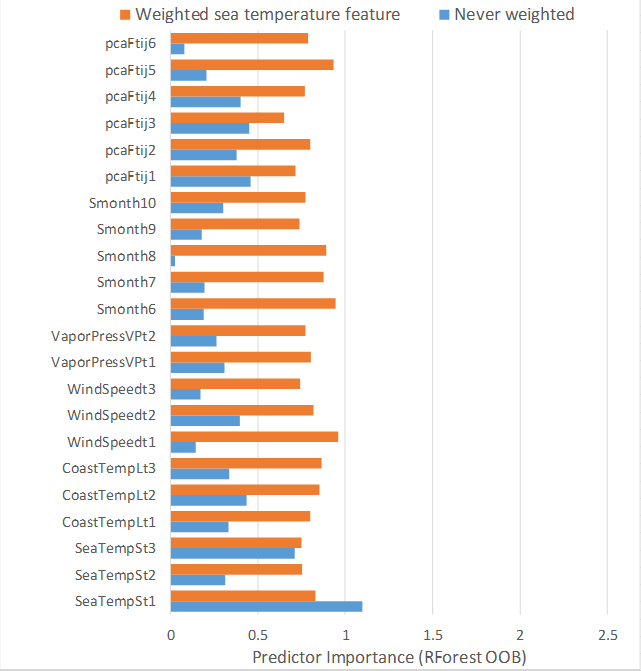}
\end{center}
\caption{Predictor Importance Comparison of Meteorological Data}\label{fig-8a}
\end{figure}

\begin{table}[h]
\caption{Set of Meteorological Data Predictors} \label{table-M2}
\begin{center}
\begin{tabular}{lp{120mm} lp{120mm}}
\textbf{NAME}  &\textbf{DESCRIPTION} \\
\hline 
SeaTemp    &$S^d_t$: Sea temperature features approximated to {3-dimension $(d=1,2,3)$ t-SNE outputs} from VGG16 of 33th fc6 layer \\
CoastTemp    &$L^k_t$: Stations$(k=1,2,3)$ observed surface temperature at coast and islands(Odawara, Oshima-kitanoyama, Kozushima) \\
WindSpeed   &$WS^k_t$: Stations observed wind speed at coast and islands, like $L^k_t$ \\
VaporPress  &$VP^m_t$: Stations$(m=1,2)$ observed vaporous pressure at mountain and coast(Mt.Fuji, Ajiro)  $(m=1,2)$ \\
Smonth &$M^s_t$: Seasonal dummy variable $\{0,1\}$, if $t$ is s-th month, then set 1, otherwise set 0, $(s=6,7,8,9,10)$ \\
pcaF  &{$PCA$-$F^{h}_t$}: Principal component analyzed to 90 percent variance of the rainfall forecast guidance with short term among {$h$-lead time $(h=1,2,3,4,5,6)$} at the dam basin \\
\hline
\end{tabular}
\end{center}
\end{table}

\begin{figure}[h]
\begin{center}
\includegraphics [width = 70mm] {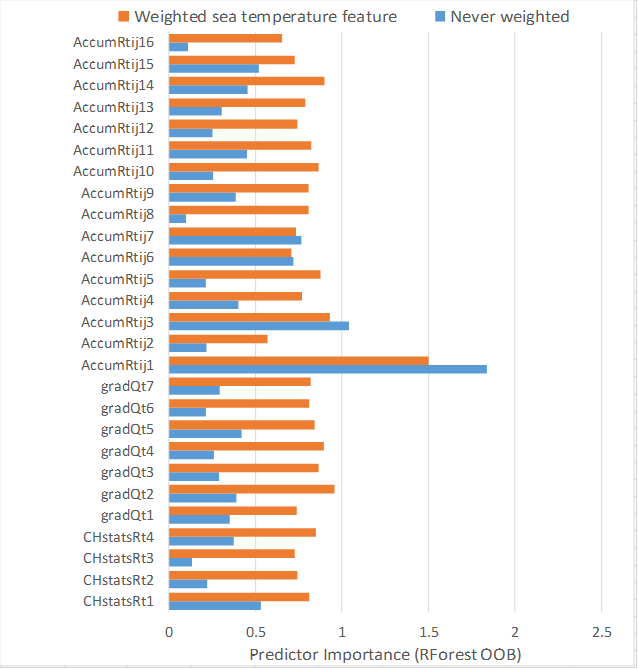}
\end{center}
\caption{Predictor Importance Comparison of Meteorological Data}\label{fig-8b}
\end{figure}

%\clearpage % 改ページ --------------------------------------------------

\begin{table}[h]
\caption{Set of Rainfall Data Predictors} \label{table-M3}
\begin{center}
\begin{tabular}{lp{120mm} lp{120mm}}
\textbf{NAME}  &\textbf{DESCRIPTION} \\
\hline 
CHstatsR &$CH(u)$-$R^{ij}_t$: $u$th moment statistics $(u=1,2,3,4)$ of analyzed rainfall  at $t$ on each 400-grids $(i,j)$ with range of 136X148 km, mean, standard deviation, skewness, and kurtosis \\
gradQ &$Q^q_t$: Gradient value of Trapezoid computed past 12-h with respect to three primary variables such as dam inflow, river water height, and ground rainfall observation at five positions $(q=1,...,7)$ \\
AccumR &$PCA^c(q)$-$R^{ij}_t$: Accumulated rainfall with range of past 6-h, using the principal components ($c=1,2,...,16$) with variance more than 85 percent$(q=0.85)$ from 400-grids \\
\hline
\end{tabular}
\end{center}
\end{table}

\vfill

\end{document}